%% file: main.tex
\documentclass[11pt]{article}

\usepackage[final]{acl}

\usepackage{times}
\usepackage{latexsym}
\usepackage{booktabs}
\usepackage{multirow} 
\usepackage{makecell}
\usepackage{textcomp}
\usepackage[T1]{fontenc}
\usepackage[utf8]{inputenc}
\usepackage{amsmath}
\usepackage[font=small]{caption}
\usepackage{xcolor} % Required for advanced color names like "Pink"
\usepackage{tcolorbox}          % For creating the colored box
\usepackage{amssymb} 
\usepackage[noend]{algpseudocode}
\usepackage[ruled,vlined]{algorithm2e}
\usepackage{microtype}

\usepackage{balance}
\usepackage{inconsolata}
\usepackage{graphicx}
\definecolor{myPink}{rgb}{1.0, 0.75, 0.8}
\definecolor{myGreen}{rgb}{0.5, 0.8, 0.5}
\definecolor{myBlue}{rgb}{0.4, 0.6, 0.9}
\definecolor{mysilver}{rgb}{0.75, 0.75, 0.75}
\usepackage{amsthm}                
\newtheorem{proposition}{Proposition}

\input{custom_commands}

\title{ART: Adaptive Reasoning Trees for Explainable Claim Verification}

\author{
    Sahil Wadhwa\textsuperscript{$\spadesuit$}, 
    Himanshu Kumar\textsuperscript{$\spadesuit$}, 
    Guanqun Yang\textsuperscript{$\clubsuit$}, \\ 
    {\bf Abbaas Alif Mohamed Nishar\textsuperscript{$\spadesuit$},} 
    {\bf Swapnil Shinde\textsuperscript{$\spadesuit$},} 
    {\bf Pranab Mohanty\textsuperscript{$\spadesuit$},} 
    {\bf Yue Wu\textsuperscript{$\spadesuit$}} \\
    \texttt{\{sahil.wadhwa,himanshu.kumar2,abbaasalif.mohamednishar,swapnil.shinde2,} \\ 
    \texttt{pranab.mohanty,yue.wu\}@capitalone.com,}
    \texttt{guanqun.yang@outlook.com} \\
    \texttt{\textsuperscript{$\spadesuit$}Capital One,}
    \texttt{\textsuperscript{$\clubsuit$}Stevens Institute of Technology}
}

\begin{document}
\maketitle

% \begingroup
%     \renewcommand\thefootnote{$\spadesuit$}
%     \footnotetext{Capital One}

%     \renewcommand\thefootnote{$\clubsuit$}
%     \footnotetext{Stevens Institute of Technology}
% \endgroup

\begin{abstract}
% The adoption of Large Language Models (LLMs) in mission-critical domains like law and medicine is hindered by their "black-box" nature. While techniques like Chain-of-Thought (CoT) prompting aim to provide reasoning, the resulting explanations are often opaque and difficult to verify. To overcome this, we propose Adaptive Reasoning Trees (ART), a novel framework that generates explains decisions. ART employs multiple LLMs to collaboratively debate and refine an initial conclusion. To ensure a rigorous debate, these LLMs can be assigned diverse personas, such as "hypercritical" or "supportive," through SFT and GRPO fine-tuning, and can ground their arguments in a knowledge base of verified cases through RAG. This tree-structured debate yields a clear, auditable reasoning process, significantly reducing the burden on practitioners by making expert verification more efficient than manual review from scratch.

Large Language Models (LLMs) are powerful candidates for complex decision-making, leveraging vast encoded knowledge and remarkable zero-shot abilities. However, their adoption in high-stakes environments is hindered by their opacity; their outputs lack faithful explanations and cannot be effectively contested to correct errors, undermining trustworthiness. In this paper, we propose \textit{ART (Adaptive Reasoning Trees)}, a hierarchical method for claim verification. The process begins with a root claim, which branches into supporting and attacking child arguments. An argument's strength is determined bottom-up via a pairwise tournament of its children, adjudicated by a judge LLM, allowing a final, transparent and contestable verdict to be systematically derived which is missing in methods like Chain-of-Thought (CoT). We empirically validate ART on multiple datasets, analyzing different argument generators and comparison strategies. Our findings show that ART's structured reasoning outperforms strong baselines, establishing a new benchmark for explainable claim verification which is more reliable and ensures clarity in the overall decision making step.

\end{abstract}

%%%%%%%%%%%%%%%%%%%%%%%%%%%%%%%%%%%%%%%%%%%%%%%%%%
% SECTIONS
%%%%%%%%%%%%%%%%%%%%%%%%%%%%%%%%%%%%%%%%%%%%%%%%%%

\input{sections/intro}
\input{sections/related_new}

\input{sections/method}

\input{sections/experiment}

\input{sections/conclusion}
\input{sections/limitations}

%%%%%%%%%%%%%%%%%%%%%%%%%%%%%%%%%%%%%%%%%%%%%%%%%%
% BIBLIOGRAPHY
%%%%%%%%%%%%%%%%%%%%%%%%%%%%%%%%%%%%%%%%%%%%%%%%%%
\bibliography{custom}

\appendix

\input{sections/appendix}

\end{document}

%% file: custom_commands.tex
\usepackage{geometry}
\usepackage{xspace}

\newcommand{\nickname}{ART\xspace}

%% file: sections/intro.tex
\section{Introduction}\label{sec:intro}

% \begin{figure}
% \centering
% \includegraphics[width=\linewidth]{figures/art.png}
% \caption{The diagram shows the proposed Adaptive Reasoning Tree (ART) system. In ART, opposing arguments are paired at each node and evaluated by a scorer. The results of these pairwise comparisons are then aggregated: the number of support wins is compared against the number of attack wins to determine the final verdict.
% In the example, the claim “Oscar Wilde could have operated a motor vehicle” is evaluated. The scorer favors two supporting arguments and one attacking argument, so the system concludes that the claim is True.}
% \end{figure}

\begin{figure*}
\centering
\includegraphics[width=\linewidth]{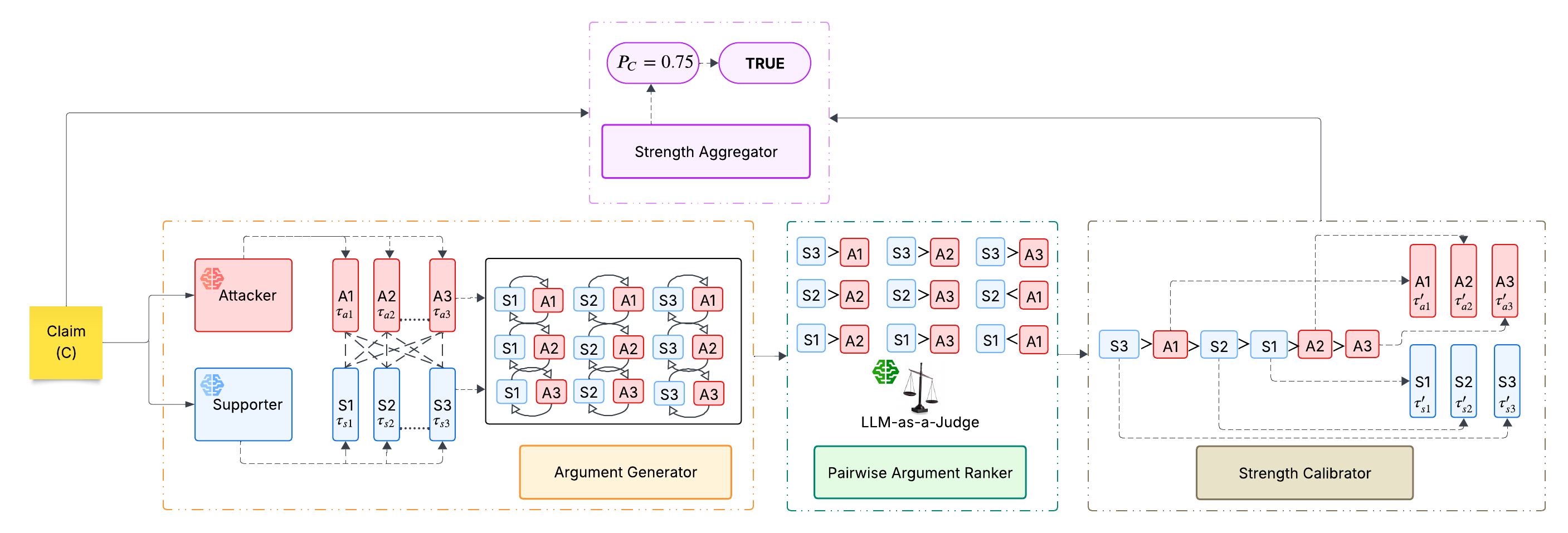}
\caption{Overview of the ART (Adaptive Reasoning Trees) framework with a tree of \textit{depth=$1$} and \textit{breadth=$3$} for illustration. Each breadth adds $b$ support and attack arguments. A claim is first processed by the \textbf{Argument Generator}, which creates a tree of supporting and attacking arguments. Arguments of opposing stances are then pitted against each other in a pairwise tournament. The \textbf{Pairwise Argument Ranker}, an LLM-as-a-judge, evaluates these pairs to dynamically update the strength of each argument. Finally, the \textbf{Strength Aggregator} consolidates these updated scores into a final probability for the claim's veracity.}
\label{fig:art_framework}
\end{figure*}

The advent of Large Language Models (LLMs) has marked a significant milestone, primarily due to their advanced reasoning capabilities~\cite{grattafiori2024llama3herdmodels, openai2024gpt4technicalreport}. This has enabled novel solutions to a wide array of complex tasks, with their impact being especially profound in specialized domains requiring nuanced judgment, such as medicine~\cite{Maity2025, 10385557} and law~\cite{kant2025robustlegalreasoningharnessing}. Despite the rapid advancement of their reasoning capabilities, a significant and persistent challenge in Large Language Models (LLMs) is their propensity for hallucination~\cite{Huang_2025, agrawal-etal-2024-language, guerreiro2023hallucinationslargemultilingualtranslation}. This phenomenon produces content that appears credible but is inaccurate or illogical. Such unreliability poses a critical barrier to their deployment in truth-sensitive applications. For instance, in automated claim verification, a hallucinating model could incorrectly validate misinformation~\cite{augenstein2023factualitychallengeseralarge}, while in generating counter-narratives to combat hate speech, it could inadvertently introduce falsehoods, undermining the very purpose of the intervention~\cite{wadhwa2024northeasternunimultilingualcounterspeech}. 
A further challenge is the inherent opacity of LLM reasoning, which conflicts with the demands of modern decision-making. Processes in critical domains require a high degree of transparency to allow for auditing, accountability, and the ability for users to contest an outcome~\cite{leofante2024contestableaineedscomputational, 10.1007/s00146-021-01251-8}. Because the internal logic of LLMs is largely inscrutable, they fail to meet this requirement, creating a fundamental barrier to their responsible deployment in high-stakes environments.

To address these challenges, we propose \textbf{\textit{ART (Adaptive Reasoning Trees)}}, a novel framework designed to enhance the transparency of automated claim verification. ART decomposes a claim into a hierarchy of supporting and attacking arguments, which then compete in a pairwise tournament adjudicated by a judge LLM. To ensure objectivity, our framework employs a dedicated LLM-as-a-Judge~\cite{rad2025refininginputguardrailsenhancing, adlakha2024evaluatingcorrectnessfaithfulnessinstructionfollowing} to perform pairwise comparisons and determine the relative strength of competing arguments. This role is deliberately separated from argument generation to prevent the inherent self-affirmation bias of a model evaluating its own outputs, which is critical for the impartiality and robustness of the final verdict.

In short, our contributions are three-fold:
\begin{itemize}
    \item \textit{ART} employs a novel tree-based structure. It uses specialized \textit{\textbf{Attack}} and \textit{\textbf{Support}} LLMs to generate arguments for and against a claim, assigning each an initial strength score.
    \item To the best of our knowledge, we are the first to introduce a tournament-style framework for claim verification where every support argument is pitted against every attack argument in a series of pairwise competitions. The outcomes are then used to dynamically update each argument's strength score via the \textit{Bradley-Terry}\footnote{\url{https://tinyurl.com/4x393csa}} model, which informs the final prediction.
    \item ART achieves state-of-the-art results on multiple claim verification datasets when tested on a variety of open-source LLMs as compared to other baselines like \textit{ArgLLM} ~\cite{freedman2025argumentativelargelanguagemodels}, Direct Prompting and Chain-of-thought (CoT)~\cite{wei2022chain}.

\end{itemize}

\textbf{Note.} The terms “claim” and “argument” are used interchangeably.

% \begin{figure}
% \centering
% \includegraphics[width=\linewidth]{example-image-a}
% \caption{Adaptive Reasoning Forests}
% \end{figure}

%% file: sections/related_new.tex
\section{Related Work}\label{sec:related}

\subsection{Explainable LLM-based Decision Making}
Large language models have demonstrated impressive reasoning abilities, but explaining their decisions in a trustworthy manner remains challenging. \textbf{Chain-of-Thought (CoT)} prompting is a common approach to induce step-by-step explanations from LLMs, yielding intermediate reasoning steps rather than just final answers~\cite{wei2022chain,kojima2022large}. While CoT can improve performance on complex tasks, its generated rationales are not guaranteed to be \textit{faithful} to the model’s actual decision process or even factually correct. Models often produce convincing but \textit{hallucinated} explanations that violate factual context~\cite{zhao2024explainability,huang2025hallucinationsurvey}. This unreliability undermines trust, especially in high-stakes domains. Moreover, slight variations in prompts or instructions can significantly alter LLM outputs, showing brittleness in instruction-following~\cite{singh2024rethinkinginterpretability}. Another challenge is evaluating the correctness of LLM-generated explanations. Human evaluation is costly, and having a model judge its own reasoning introduces bias. Recent work proposes using separate LLMs or fine-tuned modules as judges~\cite{chan2023chateval,adlakha2024evaluatingcorrectness}. While this direct usage of \textbf{LLM-as-a-Judge} paradigm improves consistency, judges drawn from the same model family may still inherit the generator’s biases unless a proper separation of roles and multiple perspectives are enforced. To address this, ART explicitly separates the roles of \textit{solver} and \textit{evaluator}: supporting and attacking arguments are generated by dedicated models, and their conflicts are adjudicated by an independent Judge LLM. 
% This aligns with the vision of \textbf{Contestable AI}~\cite{leofante2024contestableaineedscomputational}, which emphasizes explanations, grounds for contestation, redress, and interactive dialogue. 
By structuring reasoning as a tree of opposing arguments and resolving disputes through impartial judging, ART ensures that each step of the decision process can be transparently verified and, if necessary, contested.

\subsection{Automated Claim Verification}
Automated fact-checking has been studied extensively in NLP. Early approaches retrieved evidence from Wikipedia or knowledge bases and applied classifiers to label claims as supported or refuted~\cite{thorne2018fever,nakov2021factcheckingsurvey,shi2016factcheckingkb}. Retrieval-augmented generation (RAG) further improved factual grounding by conditioning on external sources~\cite{lewis2020rag,chen2023benchmarkrag}. However, these systems often provide labels without detailed justifications~\cite{Guo2021ASO}.

Recent work has turned to LLMs for claim verification, aiming to improve accuracy over the traditional three-step pipeline by leveraging their generative ability to produce justifications~\cite{Guo2021ASO,Vladika2025StepbyStepFV}.
Zero-shot or few-shot CoT prompting allows models like GPT-3/4 to generate verdicts with rationale~\cite{lee2023promptfactcheck}, but these rationales may contain hallucinations or omit critical counter-evidence~\cite{augenstein2023factualitychallenges}. 
Argument mining approaches aim to improve transparency by generating pro/con arguments~\cite{hidey2018evidence,atanasova2020generatingfactchecking}. Our work differs in that ART structures reasoning as a hierarchical argumentation tree. Each leaf is an evidence-backed argument for or against the claim, and the verdict emerges through systematic tournament-style evaluation.

\subsection{Multi-Agent Debate}
Ensembling multiple reasoning paths or models improves robustness. Self-consistency decoding aggregates multiple CoT samples to reduce individual errors~\cite{wang2022selfconsistency}. Multi-agent systems extend this by having multiple LLMs collaborate or compete. The idea of AI debate was introduced by ~\cite{irving2018aidebate}, where two agents argue and a judge makes the decision. Recent implementations show that multi-agent debate improves factual accuracy~\cite{du2023multiagentdebate,khan2024debatepersuasive}. Debate has also been used for evaluation, e.g., ChatEval~\cite{chan2023chateval} and Debate Helps~\cite{michael2023debatehelps}.

Another line of work uses specialized roles or personas for different LLMs. Wang et al.~\cite{wang2023synergy} explored multi-persona collaboration, while Jung et al.~\cite{jung2022critic} proposed critic LLMs to identify flaws. These role-based ensembles demonstrate the benefit of diversity in reasoning styles. ART builds on this idea with distinct \textit{Supporter}, \textit{Attacker}, and \textit{Judge} roles, and introduces a structured tournament that ensures balanced evaluation between opposing views.

\subsection{Comparison with ArgLLMs}
\label{sec:comparison_argllm}
Our work is related to the recent proposal of \textit{Argumentative LLMs (ArgLLMs)}~\cite{freedman2025argumentativelargelanguagemodels}, which embeds LLM-generated pro/con arguments into formal \textit{quantitative bipolar argumentation frameworks} (QBAFs). Within this paradigm, argument strengths are fixed and aggregated using deterministic semantics (e.g., DF-QuAD), yielding provable properties of contestability and logical soundness. This line of work makes an important contribution by grounding LLM outputs in formal argumentation theory. ART, rather than imposing static semantics, adopts a \textit{dynamic pairwise tournament} mechanism: every support argument is directly compared against every attack argument, with outcomes adjudicated by an independent Judge LLM. This shift enables reasoning to adapt flexibly to the relative persuasiveness of arguments, rather than relying on pre-defined aggregation rules.  Moreover, ART enforces \textit{role specialization} by explicitly separating the Supporter, Attacker, and Judge models. This avoids the \textit{self-affirmation bias} inherent in ArgLLMs, where a single model family is often reused for both argument generation and evaluation. By diversifying roles across models, ART enhances impartiality and robustness. 

In summary, ART builds on the intuition that argumentative prompting improves transparency, but diverges from ArgLLMs by introducing a dynamic, role-separated, and empirically validated framework. This makes ART better suited for real-world claim verification tasks, where adaptability and benchmarked performance are as critical as theoretical guarantees.

%% file: sections/method.tex
\section{Methodology}\label{sec:method}
% TODO
% content
% - ensemble of multiple ART trees with depths up to 3 or 4.
% - replace 70B scoring with pairwise comparison. 
% - not rely on uncertainty.py
% a diagram that shows that the proposed system reduces bias and variance

\nickname structures the verification process as a binary reasoning tree in which the root node, referred to as the \emph{claim node}, represents the input claim under examination, and all other nodes are \emph{argument nodes} that recursively support or attack their parent. 
This claim verification process is explicit with each comparison in the tree contributing to the overall decision.  
\nickname operates in four main phases: \emph{argument generation}, \emph{pairwise argument ranking}, \emph{strength calibration} and \emph{strength aggregation}. Figure \ref{fig:art_framework} presents an overview of ART showing an example tree with a depth of one and a breadth of three.

\subsection{Argument Generation} 
Each argument node is instantiated by a debater LLM prompted with a specific persona of either \emph{supporter} (${S}$) or \emph{attacker} (${A}$). 
The supporter produces arguments in favor of the parent node, while the attacker generates counterarguments against it. 
Although in principle the reasoning process could grow without bound, we cap the maximum tree depth at \textit{two} in this work to balance reasoning depth and computational efficiency. \textbf{Note:} Same LLM is used for both support and attack. 

\subsection{Pairwise Argument Ranking} 
Once generated, arguments are evaluated through a comparison procedure by an LLM-as-a-judge ($J$)~\cite{rad2025refininginputguardrailsenhancing, adlakha2024evaluatingcorrectnessfaithfulnessinstructionfollowing}.
For each pair of argument nodes, the scoring LLM selects which of the supporting or attacking child argument is more persuasive given the parent node argument.
A complete traversal of the tree yields $N = \frac{(2b)^{d+1} - 1}{2b - 1}$ nodes of depth $d$ and breadth $b$ with $\frac{(2b)^{d} - 1}{2b-1}*b^2$ pairwise comparisons (see Appendix ~\ref{sec:time_complexity}). 
Importantly, \nickname is \emph{adaptive}: subtrees whose comparisons fall below a confidence threshold can be selectively pruned during \emph{argument verdict} to improve the accuracy of the final decision.

Let $p$ be a parent argument node in the tree, with $k$ supporting and attacking children defined as the sets:
$\mathcal{S}(p) = \{S_1, S_2, \dots, S_k\}$, and
$\mathcal{A}(p) = \{A_1, A_2, \dots, A_k\}$.

For each pair $(a, b)$ where $a \in \mathcal{S}(p)$ and $b \in \mathcal{A}(p)$, a scoring LLM, $J$, selects the more persuasive argument:
$$
\text{outcome}(a, b) = J(a, b \mid p) \in \{a, b\}
$$

\subsection{Strength Calibration}
\label{subsection:strength_calibration}
For each argument $i \in \mathcal{S}(p) \cup \mathcal{A}(p)$, the judge $J$ also calculates its intrinsic strength, denoted $\tau_i$,  referring to how strongly its supports or attacks its parent claim similar to ~\cite{freedman2025argumentativelargelanguagemodels}. 
While an intrinsic strength score, $\tau_i$, provides a baseline measure for an argument, this absolute evaluation, performed in isolation by a judge model $J$, is insufficient for a robust analysis. Such isolated scores struggle to handle redundancy, are susceptible to inconsistent scaling and bias from the judge model, and fail to establish a clear relative ranking.

In contrast, pairwise scoring addresses these limitations by reframing the task as a more constrained and reliable relative judgment: determining which of two arguments is stronger, a simpler task for an LLM~\cite{carterette2008here}. This comparative approach naturally produces a robust ranking that surfaces nuanced qualitative differences and accounts for semantic overlap, which absolute scoring methods cannot capture. This presents a trade-off between the fast but potentially biased intrinsic scores and the robust but computationally expensive pairwise comparisons. To balance this, our approach uses the \textbf{Bradley-Terry (BT) model to calibrate the initial intrinsic strengths}. The BT model analyzes the tournament outcomes to derive a purely relative performance score for each argument. This data-driven score then acts as a scaling factor, adjusting the initial intrinsic value $\tau_a$ up or down based on its competitive success. This hybrid approach retains the baseline evaluation while refining it with robust, comparative data.

\paragraph{Bradley--Terry Calibration.}
Each child argument $i\in \mathcal{S}(p)\cup \mathcal{A}(p)$ has a latent strength $\theta_i>0$. For $a\in\mathcal{S}(p)$ and $b\in\mathcal{A}(p)$, the probability that $a$ is judged stronger than $b$ is
\begin{equation}
P(a \succ b) \;=\; \frac{\theta_a}{\theta_a + \theta_b}
\label{eq:bt_prob}
\end{equation}

Let $U(p)=\mathcal{S}(p)\cup \mathcal{A}(p)$. We estimate $\theta$ by fixed-point iteration using only cross-pairs $(a,b)\in \mathcal{S}(p)\times \mathcal{A}(p)$, then normalize:

\begin{equation}
\label{eq:bt_update}
\begin{aligned}
\theta_a^{(t+1)} &=
\frac{\sum_{b \in \mathcal{A}(p)} \tau_{ab}}
     {\sum_{b \in \mathcal{A}(p)} 
      \frac{\tau_{ab} + \tau_{ba}}{\theta_a^{(t)} + \theta_b^{(t)}} + \varepsilon},
&& a \in \mathcal{S}(p),\\[2pt]
\theta_b^{(t+1)} &=
\frac{\sum_{a \in \mathcal{S}(p)} \tau_{ba}}
     {\sum_{a \in \mathcal{S}(p)} 
      \frac{\tau_{ab} + \tau_{ba}}{\theta_a^{(t)} + \theta_b^{(t)}} + \varepsilon},
&& b \in \mathcal{A}(p),\\[2pt]
\sum_{i \in U(p)}\theta_i^{(t+1)} &= 1 \ \text{s.t.}\ \theta_i \in [0,1]
\end{aligned}
\end{equation}

\noindent Here, $\tau_{ab}$ counts judgments where supporter $a$ is stronger than attacker $b$ (and $\tau_{ba}$ is the reverse), and $\varepsilon>0$ prevents division by zero.
% \begin{equation}
% {
% \begin{aligned}
% \tau_i' &= \mathrm{clip}_{[0,1]}\!\bigl(\tau_i \cdot |U| \, \cdot \theta_i \bigr), && i \in U(p)\\[1pt]
% \end{aligned}
% }
% \label{eq:bt_scaled}
% \end{equation}
Eq.~\ref{eq:bt_update} is an \textit{Majorize–Minimize} (MM) step for the Bradley–Terry log-likelihood, ensuring monotone ascent and scale-invariance to pairwise counts (duplicates or batch reweighting leave strengths unchanged up to normalization). These properties justify early stopping and yield stable, reproducible calibration (Proposition~\ref{prop:mm_bt}; Algorithm~\ref{alg:calibrate_bt_clearW}).
{
% \small
\begin{equation}
{
\begin{aligned}
\tau_i' = \mathrm{clip}_{[0,1]}\!\bigl((1-\lambda)\,\tau_i + \lambda\,\theta_i\bigr),\\
i\in U(p), \lambda\in[0,1]
\end{aligned}
}
\label{eq:bt_scaled}
\end{equation}
}

We update each argument’s weight by \emph{blending} its current value with its BT score (Eq.~\ref{eq:bt_scaled}). Specifically, the blend parameter $\lambda\in[0,1]$ controls how much we trust the BT score—larger $\lambda$ moves the weight closer to $\theta_i$—and we clip the result to $[0,1]$ to avoid extremes. This convex update keeps values bounded, smooths noisy estimates, and still reflects the pairwise outcomes.

\subsection{Strength Aggregation}
The outcomes of the local comparisons are aggregated upward through the tree until a final decision is reached at the claim node. 
This final stage produces both a binary classification label and an interpretable reasoning trace, showing how the balance of supporting and attacking arguments led to the system’s conclusion.

\paragraph{Final Claim Score.}
After BT calibration step, the strength of each argument $p$ is updated in topological order. 
For $p$ with attackers $\mathrm{A}(p)$ and supporters $\mathrm{S}(p)$, the ProductAggregation term is calculated as
\begin{equation}
\alpha_p \;=\;
\prod_{a \in \mathrm{S}(p)} (1 - s_a)
\;-\;
\prod_{b \in \mathrm{A}(p)} (1 - s_b),
\label{eq:alpha_pre}
\end{equation}
where $s_a$ and $s_b$ are the propagated strengths of the respective children. 

Ultimately, the final score of the parent argument node becomes
\begin{equation}
s_p \;=\;
\begin{cases}
\tau_p' - \alpha_p* \tau_p', & \alpha_p > 0, \\[2pt]
\tau_p' - \alpha_p * (1 - \tau_p'),       & \alpha_p \leq 0,
\end{cases}
\label{eq:linfluence_pre}
\end{equation}
where $\tau_p'$ is the BT-calibrated intrinsic weight of $p$.

The score of the root claim $C$ is calculated as
\begin{equation}
P(C) \;=\; s_C
\label{eq:final_claim}
\end{equation}

Argument strengths propagate via a DF\mbox{-}QuAD–style quantitative semantics~\cite{ragodfquad}:
attacker/supporter effects are aggregated (Eq.~\ref{eq:alpha_pre}) and node strengths updated
(Eq.~\ref{eq:linfluence_pre}) deterministically, yielding smooth, discontinuity\mbox{-}free propagation (see Proposition~\ref{prop:dfquad_proof}).
The final claim score \(P(C)\) (Eq.~\ref{eq:final_claim}) is labeled \textbf{True} if \(P(C)>0.5\), else \textbf{False}. We set the classification \emph{threshold} to be $0.5$.

%% file: sections/experiment.tex
\section{Experiments}\label{sec:exp}
% TODO
\subsection{Datasets}
\paragraph{MedQA}~\cite{medqa} dataset comprises multiple-choice medical examination questions in English, Simplified Chinese, and Traditional Chinese.
Each example contains a clinical vignette, a question, and five answer choices, with exactly one correct answer.
We utilize the test split of the English subset.
Since the original dataset is in multiple-choice format with answer choices typically expressed as short phrases, prior work has converted these into declarative statements: each answer choice is rewritten as a complete sentence, with the statement corresponding to the correct answer labeled as true~\cite{freedman2025argumentativelargelanguagemodels}.
The processed test set contains $250$ positive and $250$ negative examples.

\paragraph{StrategyQA}~\cite{strategyqa} dataset consists of open-domain questions requiring multi-hop reasoning over common sense and factual knowledge.
Unlike MedQA, each example in StrategyQA contains only a question without additional context, with answers provided as simple phrases.
Similar to MedQA, prior work has processed StrategyQA by converting each question into a declarative statement with an associated truth label~\cite{freedman2025argumentativelargelanguagemodels}.
The processed test set is balanced with $250$ positive and $250$ negative examples.

\paragraph{TruthfulQA} ~\cite{lin2022truthfulqameasuringmodelsmimic} dataset comprises of $250$ positive and $250$ negative claims respectively and is specifically used to evaluate LLMs to gauge if they can correctly identify truthful answers without being deceived by common misconceptions and falsehoods.

\begin{figure*}
\centering
\includegraphics[width=\linewidth]{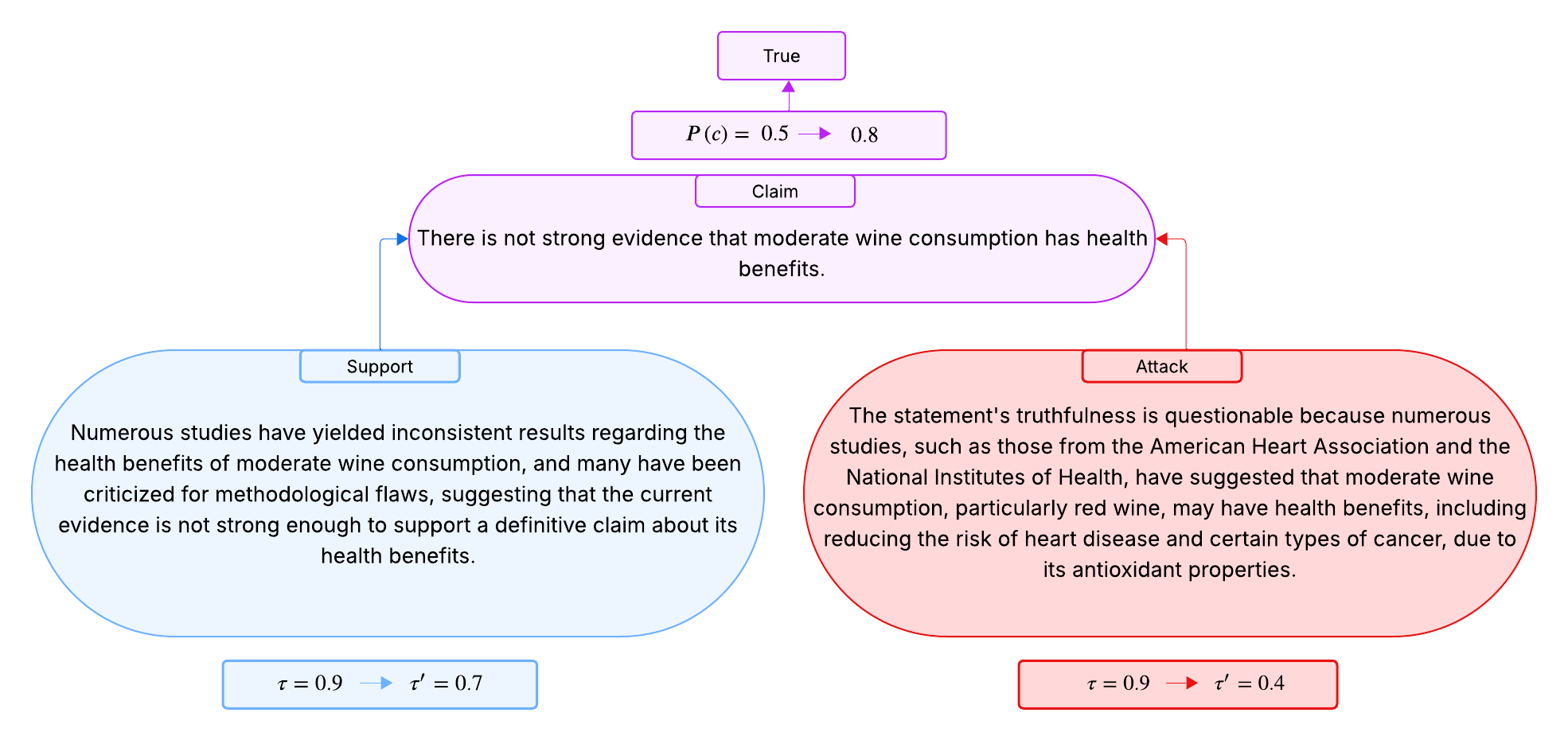}
\caption{Example from \textit{MedQA} with \textit{Llama 3.3 70B Instruct}. Pairwise tournaments calibrate argument strengths by blending intrinsic weights $\tau$ with BT scores $\theta$ to yield $\tau'$. Here, support and attack start with equal intrinsic weights but calibrate to $0.7$ and $0.4$, shifting the claim’s probability from $0.5$ (False) to $0.8$ (True). Without the pairwise and BT calibration steps, the claim would be misclassified as False.}
\label{fig:qual_ex1}
\end{figure*}

\subsection{Baselines and Models}
\paragraph{Models.} 
ART is flexible to the choice of underlying language model. In this paper, we focus on widely used open-source LLMs that combine strong performance across diverse tasks with easy replicability of results. 
Specifically, we experiment with \textit{Llama~3.1 3B-Instruct} and \textit{Llama~3.1 8B-Instruct}, as well as the larger \textit{Llama~3.3 70B-Instruct}~\cite{grattafiori2024llama3herdmodels}, 
together with \textit{Mistral-7B Instruct-v0.2}~\cite{jiang2023mistral7b} and \textit{Mixtral-8x7B-Instruct-v0.1}~\cite{jiang2024mixtralexperts}. 
Unless otherwise specified, all models are queried with the following decoding parameters: 
\texttt{"temperature": 0.2, "max\_new\_tokens": 512, "top\_p": 0.95}. All models are served through the OpenAI-compatible vLLM\footnote{\url{https://docs.vllm.ai/en/latest/}} API and batch executed on 1-4 NVIDIA A100 GPUs depending on model sizes.

\paragraph{Baselines.}
We evaluate our approach against three baselines. 
\textbf{Direct Prompting}: the LLM is directly prompted to verify a claim with a \textit{Yes/No} answer. 
\textbf{Chain-of-Thought (CoT)}~\cite{wei2022chain}: the LLM is additionally asked to generate reasoning alongside its Yes/No decision. 
\textbf{ArgLLM}~\cite{freedman2025argumentativelargelanguagemodels}: this method constructs a tree-like structure but without pairwise comparisons or calibrations. The differences between ART and ArgLLM are discussed in Section~\ref{sec:comparison_argllm} in more detail. All the prompts used for Direct Prompting, CoT, arguments generation, intrinsic strength calculation, pairwise comparisons, are shown in Appendix ~\ref{appendix:prompts}.

% In your preamble add:
% \usepackage{makecell}

\begin{table*}[t]
  % \small
  \centering
  \resizebox{\linewidth}{!}{%
    \begin{tabular}{ll|cc|ccc|ccc}
      \toprule
      \textbf{Dataset} & \textbf{Model} & \textbf{Direct} & \textbf{CoT} &
      \makecell{\textbf{ArgLLM}\\(Eval=70B)} & 
      \makecell{\textbf{ART}\\(w/o intr.)\\(Eval=70B)} &
      \makecell{\textbf{ART}\\(intr.)\\(Eval=70B)} &
      \makecell{\textbf{ArgLLM}\\(Eval=Self)} & 
      \makecell{\textbf{ART}\\(w/o intr.)\\(Eval=Self)} &
      \makecell{\textbf{ART}\\(intr.)\\(Eval=Self)} \\
      \midrule
      \multirow{5}{*}{MedClaim} 
        & Llama 3.1 3B Instruct      & 51.4 & 57.9 & 58.7  & \underline{63.1} & \textbf{65.0} & 52.2 & \textbf{54.3} & \underline{53.7} \\
        & Llama 3.1 8B Instruct      & 54.5 & \textcolor{blue}{64.6} & 62.9 & \textbf{64.6} & \underline{63.6} & \underline{54.9} & 53.7 & \textbf{56.0} \\
        & Mistral-7B Instruct-v0.2   & 51.4 & 50.4 & 54.3 & \underline{54.8} & \textbf{57.7} & 47.1 & 50.1 & \textbf{52.8} \\
        & Mixtral-8x7B Instruct-v0.1 & 58.7 & 53.5 & \underline{62.3} & \textbf{65.3} & 61.7 & \textbf{61.5} & 56.2 & \underline{57.7} \\
        & Llama 3.3 70B Instruct     & \textcolor{blue}{76.0} & 74.3 & 67.7 & \underline{68.8} & \textbf{69.0} & -- & -- & -- \\
      \midrule
      \multirow{5}{*}{StrategyQA} 
        & Llama 3.1 3B Instruct      & 58.6 & \textcolor{blue}{67.0} & 61.4 & \underline{62.4} & \textbf{62.6} & 55.6 & \textbf{58.6} & \underline{57.6} \\
        & Llama 3.1 8B Instruct      & 58.2 & \textcolor{blue}{73.6} & 64.4 & \underline{69.8} & \textbf{71.4} & \underline{58.8} & 57 & \textbf{59.2} \\
        & Mistral-7B Instruct-v0.2   & 59.4 & 55.4 & \underline{65.6} & 64.4 & \textbf{70.0} & 55.2 & \textbf{58.6} & \underline{58.4} \\
        & Mixtral-8x7B Instruct-v0.1 & 55.6 & 68.4 & 64.4 & \textbf{70.2} & \underline{68.8} & \underline{62} & 61.2 & \textbf{63.6} \\
        & Llama 3.3 70B Instruct     & 81.0 & \textcolor{blue}{82.4} & 68.4 & \underline{74.2} & \textbf{75.8} & -- & -- & -- \\
        \midrule
         \multirow{5}{*}{TruthfulQA} 
        & Llama 3.1 3B Instruct      & 62.2 & \textcolor{blue}{68} & 65.2 & \underline{66.9} & \textbf{67.0} & 60.6 & \textbf{65.0} & \underline{63.8} \\
        & Llama 3.1 8B Instruct      & 63.0 & 70.4 & 67.8 & \underline{69.4} & \textbf{75.6} & \underline{59.8} & 52.8 & \textbf{68.2} \\
        & Mistral-7B Instruct-v0.2   & 70.6 & 57.2 & \underline{68.2} & 65.6 & \textbf{69.2} & \underline{57.2} & 54.4 & \textbf{59.8} \\
        & Mixtral-8x7B Instruct-v0.1 & 66.6 & 62.8 & 71.4 & \underline{76.4} & \textbf{78.0} & \textbf{62.8} & 60.4 & \underline{61.6} \\
        & Llama 3.3 70B Instruct     & 81.6 & \textcolor{blue}{83.6} & 68.4 & \textbf{80.6} & \underline{79.4} & -- & -- & -- \\
      \bottomrule
    \end{tabular}
  }
  \caption{Evaluation results on MedClaim, StrategyQA and TruthfulQA comparing Direct prompting, CoT, ArgLLM, and ART with \textit{depth=$1$} and \textit{breadth=$1$}. 
Results are shown for two evaluator settings: (i) \textit{Eval = 70B}, where judgments are made by \textit{Llama 3.3 70B Instruct}, and (ii) \textit{Eval = Self}, where the evaluator matches the main model. 
The \textit{w/o intr.} setting uses only the BT score $\theta$ by setting $\lambda=1$, while \textit{intr.} combines $\theta$ with intrinsic strength $\tau$ with $\lambda=0.5$ (Eq.~\ref{eq:bt_scaled}). See Section~\ref{subsection:strength_calibration} for details. Bold indicates the best \textit{Accuracy} score, while underlining marks the second best within the two evaluator settings. \textcolor{blue}{Blue} color text indicates cases where CoT or Direct achieves the best performance.} 
  \label{tab:results_both}
\end{table*}

\section{Results}
\label{sec:results}
We benchmark ART against Direct prompting, CoT prompting, and ArgLLM~\cite{freedman2025argumentativelargelanguagemodels} in Table~\ref{tab:results_both}, using tree depth and breadth fixed at 1 across datasets and LLMs. In this ART configuration, each root claim is evaluated through a single supporting argument and a single attacking argument competing to determine the claim’s validity.

\paragraph{Evaluator Settings:} We evaluate two evaluator settings. In the \textit{Single-LLM} setup (\textit{Eval=Self}), the same LLM performs argument generation, intrinsic strength computation, pairwise comparisons, and final judgment. This setup directly tests the model’s ability to identify relatively stronger arguments, but it is also prone to self-affirmation bias since generation and adjudication rely on the same LLM. In contrast, the \textit{Multi-LLM} setup (\textit{Eval=70B}) separates these roles: one model generates arguments, while a stronger judge LLM performs pairwise comparisons and intrinsic strength calculation for ART nodes, thereby reducing the bias present in the \textit{Single-LLM} configuration. We use \textit{Llama 3.3 70B-Instruct} model as the judge model due to its strong performance in complex problems. The \textit{Multi-LLM} configuration delivers substantially better results than \textit{Single-LLM} in our studies.

\paragraph{Bradley--Terry Calibration with Intrinsic Strengths:}
We employ the Bradley--Terry (BT) model to derive argument strengths from pairwise tournaments. In the \textit{w/o intr.} setting, $\lambda$ is set to $1$ and only the BT-derived score $\theta$ is used (Eq.~\ref{eq:bt_scaled}) to achieve the final score $s$ (Eq. \ref{eq:final_claim}), while in the \textit{intr.} setting, $\lambda$ is set to $0.5$ so that the BT score is further blended with the intrinsic strength $\tau$, yielding the adjusted value $\tau'$ to calculate $s$. We believe intrinsic strength helps regularize the BT estimates by encoding prior plausibility/quality, reducing variance when pairwise evidence is sparse or noisy, and preventing brittle swings from a few contested comparisons. On average, we achieve a $3\%$ relative accuracy improvement in the \textit{intr.} setting over the \textit{w/o intr.} setting having aggregated over both the evaluator settings. \textbf{Note.} If not explicitly specified, ART should be considered with the \textit{intr.} and \textit{Multi-LLM} (\textit{Eval=70B}) configuration.

\subsection{Comparison with Direct and CoT}
We observe that ART outperforms the Direct Prompting methodology overall, which is expected since direct prompting requires the LLM to evaluate the claim in a single step without any structured decomposition of arguments. In contrast, ART explicitly organizes the reasoning process into supporting and attacking arguments, enabling more reliable adjudication of the root claim. 
The main gap appears on larger models: CoT often performs better because they exploit rich step-by-step reasoning, while ART’s multiple argument branches can add noise and slightly reduce accuracy. Overall, while ART not only achieves at par performance over CoT, its key advantages lie beyond accuracy. ART produces faithful, contestable explanations structured as arguments, from which final decisions are deterministically derived. By contrast, CoT traces offer post-hoc rationalizations that may not reflect the model’s actual stochastic reasoning process~\cite{turpin2023languagemodelsdontsay}. ART therefore provides greater transparency and verifiability, as each decision can be directly traced to its underlying arguments and semantics.

\subsection{ART vs ArgLLM}
When $\lambda=0$, the update reduces to the ArgLLM baseline: no BT contribution is applied and each weight stays at its pre-BT value (i.e., $\tau_i'=\tau_i$). This recovers the original ArgLLM setting and serves as a clean ablation point for comparisons. We observe that in most cases (see Table~\ref{tab:results_both} and ~\ref{tab:argllm_art_variants}), ART outperforms ArgLLM. This can be attributed to how strength is computed for each node: ArgLLM estimates strengths in isolation, without considering sibling arguments, which limits the robustness of its scoring. In contrast, ART employs a pairwise tournament in which arguments of opposite polarity directly compete, with strengths estimated through the Bradley--Terry (BT) model. This ensures that scores are both relative and comparable within the same latent space, producing more reliable outcomes. Furthermore, ART’s calibration step grounds final decisions in argumentation semantics, enhancing consistency and interpretability compared to ArgLLM’s uncalibrated node-level scores. ArgLLM is a special case of ART obtained by setting $\lambda=0$: the update ignores pairwise outcomes and applies no score calibration, relying solely on the intrinsic strengths of the arguments.

Figure~\ref{fig:qual_ex1} shows a qualitative example where two arguments that begin with equal intrinsic weights can calibrate to $0.7$ (support) versus $0.4$ (attack), shifting the root claim’s probability from $0.5$ to $0.8$ and flipping a mistaken rejection into a correct acceptance. In practice, this pairwise-and-calibrate step addresses \texttt{ArgLLM}’s main failure mode—uncalibrated, non-comparable scores—while preserving transparent, contestable reasoning traces. ART moves all arguments into a shared latent scale and corrects over/under-weighting.

\begin{figure}
\centering
\includegraphics[width=\linewidth]{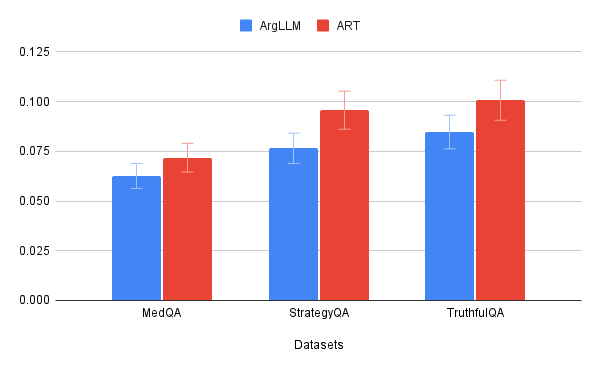}
\caption{Variance of strengths distribution under ArgLLM ($\lambda=0$) vs.\ ART ($\lambda=0.5$). of the root claim nodes. ART yields higher variance—scores pushed away from $0.5$—indicating more decisive, calibrated strengths.}
\label{fig:variance_comparison}
\end{figure}

\paragraph{Bias Variance Tradeoff:}

We observe that BT--calibrated strengths lead to higher variance in the final claim prediction probabilities as compared to ArgLLM strengths (see Figure~\ref{fig:variance_comparison}). ART increases variance by decompressing intrinsic scores and incorporating BT comparison evidence, but improves ranking calibration. ART helps with bias reduction via pairwise comparisons as opposed to isolated strength calculations.

\subsection{Tree Variations}
We analyze performance across different tree sizes in terms of depth and breadth in Table~\ref{tab:argllm_art_variants}. 
To balance accuracy with latency, we restrict our exploration to a maximum depth and breadth of two. 
Our results show that increasing the breadth of the tree improves the performance of models with different size across benchmarks. Overall, breadth makes more difference than the height as each argument in the tree is now validated with more support and attack child nodes giving more context to the evaluator model $J$. ART outperforms ArgLLM in this setup as well.

Having shown that widening a single tree is effective, we next examined whether an ensemble of separate, simpler trees could be a competitive alternative. We evaluate this multi-tree ensemble—along with its computational trade-offs—in Appendix~\ref{sec:appendix_multi_tree}. Empirically, while greater breadth for pairwise comparisons remains a strong general-purpose strategy, a minimalist configuration (\(D{=}1, B{=}1\)) coupled with a strong evaluator (e.g., Llama~3.3~70B) can be unexpectedly potent, particularly under compute constraints.

\begin{table}[t]
  \centering
  \resizebox{\columnwidth}{!}{%
    \begin{tabular}{llcccc}
      \toprule
      \textbf{Dataset} & \textbf{Model} & 
      \makecell{\textbf{ArgLLM}\\($D{=}1,B{=}2$)} & 
      \makecell{\textbf{ArgLLM}\\($D{=}2,B{=}1$)} & 
      \makecell{\textbf{ART}\\($D{=}1,B{=}2$)} & 
      \makecell{\textbf{ART}\\($D{=}2,B{=}1$)} \\
      \midrule
      \multirow{5}{*}{MedClaim} 
        & Llama 3.1 3B Instruct      & \underline{57.9} & 56.8 & \textbf{60.8} & 53.3 \\
        & Llama 3.1 8B Instruct      & \underline{62.5} & 61.9 & \textbf{62.9} & 53.7 \\
        & Mistral-7B Instruct-v0.2   & \underline{55.6} & 52.8 & \textbf{58.3} & 50.1 \\
        & Mixtral-8x7B Instruct-v0.1 & \textbf{66.1} & 62.7 & \underline{65.7} & 55.8 \\
        & Llama 3.3 70B Instruct     & \underline{67.1} & \underline{67.1} & \textbf{67.4} & 56.0 \\
      \midrule
      \multirow{5}{*}{StrategyQA} 
        & Llama 3.1 3B Instruct      & \underline{63.2} & 61.8 & \textbf{64.4} & 57.6 \\
        & Llama 3.1 8B Instruct      & \underline{64.8} & 66.2 & \textbf{72.2} & 53.4 \\
        & Mistral-7B Instruct-v0.2   & 66.8 & \underline{67.0} & \textbf{70.2} & 59.2 \\
        & Mixtral-8x7B Instruct-v0.1 & 62.6 & \textbf{67.0} & \underline{65.2} & 54.4 \\
        & Llama 3.3 70B Instruct     & \underline{69.4} & 68.2 & \textbf{75.4} & 58.2 \\
        \midrule
       \multirow{5}{*}{TruthfulQA} 
        & Llama 3.1 3B Instruct      & \underline{68.0} & 67.4 & \textbf{69.0} & 60.8 \\
        & Llama 3.1 8B Instruct      & \underline{69.6} & 67.0 & \textbf{73.0} & 55.0 \\
        & Mistral-7B Instruct-v0.2   & \underline{69.6} & 68.6 & \textbf{72.2} & 51.4 \\
        & Mixtral-8x7B Instruct-v0.1 & 72.4 & \underline{76.0} & \textbf{79.4}& 62.2 \\
        & Llama 3.3 70B Instruct     & 71.0 & \underline{74.8} & \textbf{80.2} & 63.0 \\ 
      \bottomrule
    \end{tabular}
  }
  \caption{\small Accuracy of ArgLLM and ART on datasets with different tree settings (\textit{Eval=70B}, \textit{intr.}). Best scores are in bold and second best are underlined. \textbf{D} = depth, \textbf{B} = breadth.}
  \label{tab:argllm_art_variants}
\end{table}

%% file: sections/conclusion.tex
\section{Conclusion and Future Work}
\label{sec:conclusion}
We introduce Adaptive Reasoning Trees (ART) as a principled framework for transparent and contestable claim verification. ART systematically decomposes a claim into a hierarchy of supporting and attacking arguments, forming a tree structure where each argument may itself branch into further pro and con arguments. These arguments engage in pairwise tournaments, with their relative strengths estimated using the Bradley–Terry model, ensuring that scores are calibrated within a shared latent space. Through a bottom-up aggregation process, ART computes the overall strength of the root claim to decide whether it should be accepted or rejected. ART is robust across benchmarks and configurations (varying tree breadth and height), producing faithful, interpretable traces and consistently outperforming Direct Prompting, Chain-of-Thought, and ArgLLM in accuracy and robustness. It delivers a transparent, contestable verification framework in which final decisions arise from the balanced aggregation of all arguments.

Future work can explore grounding-based argument generation using knowledge bases like \textit{Wikipedia} for evidence-backed support and attack generation. Additionally, tournaments could weigh comparisons by evidence quality, recency, and source reliability for provenance-aware, auditable decisions. Moreover, while Bradley--Terry offers a solid calibration method, sophisticated methods leveraging neural networks that can learn the nonlinear cues and interactions could be a potential avenue. Lastly, ART could be augmented with a reinforcement learning objective that rewards each component for the correct collective decision.

%% file: sections/limitations.tex
\section{Limitations}\label{sec:limitations}
While ART offers a transparent, contestable and traceable framework for claim verification, it has some limitations. As compared to Direct and CoT prompting methods that make a single LLM call, ART makes multiple LLM calls - for support/attack generation, intrinsic scoring, and pairwise comparisons which could lead to higher time complexity. However, with multiple parallel asynchronous calls, this time could be brought down significantly. Additionally, LLM calls could be reduced by optimizing LLMs to generate both support and attacks at the same time. While ART offers good results on the three datasets used in this paper, it may still face challenges with claims that are more complex and uncertain and require extensive world knowledge, and domain expertise. Such cases can amplify error propagation across the tree and stress the judge’s calibration. LLM judges can be prompt sensitive or biased towards a certain set of arguments, which could reduce the reliability on fair judgment and make the overall process less trustworthy. Despite these caveats, ART lays the groundwork for claim verification—fostering explainable, auditable, and contestable reasoning.

%% file: sections/appendix.tex
\section*{Appendix}

\section{ART Algorithm}
ART is leveraged for objectively calibrating the persuasive weights of arguments within a hierarchical tree. ART identifies points of contention where a claim is defended by "supporters" and challenged by "attackers." It then leverages a large language model (LLM) as an impartial judge to conduct pairwise comparisons between every competing supporter and attacker. By aggregating the results of these head-to-head contests, the algorithm uses the Bradley-Terry statistical model to compute a "strength score" for each argument based on its performance. This score is then used to intelligently update the argument's initial weight, ensuring the final structure reflects the relative quality of its components as determined by the LLM judge.

\section{Bradley--Terry}

% Assumes amsthm is set up:
% \usepackage{amsthm,amsmath}
% \newtheorem{proposition}{Proposition}

\noindent\textbf{MM view and notation (bipartite A--B).}
Let $A$ be supporters, $B$ attackers, and $E:=A\times B$ the set of cross-pairs.
For $(a,b)\in E$, let $\tau_{ab}$ be wins of $a$ over $b$, $n_{ab}=\tau_{ab}+\tau_{ba}$, and
$S^{(t)}_{ab}=\theta^{(t)}_a+\theta^{(t)}_b$ (pair-sum at iterate $t$).
Define
\[
\begin{aligned}
\text{wins}_a &= \sum_{b\in B}\tau_{ab}, \qquad
\text{wins}_b &= \sum_{a\in A}\tau_{ba}, \\
c_a^{(t)} &= \sum_{b\in B}\frac{n_{ab}}{S^{(t)}_{ab}}, \qquad
c_b^{(t)} &= \sum_{a\in A}\frac{n_{ab}}{S^{(t)}_{ab}}.
\end{aligned}
\]

Linearizing $-\log(\theta_a+\theta_b)$ at $S^{(t)}_{ab}$ yields a separable lower bound,
maximized at $\tilde\theta_u=\text{wins}_u/c^{(t)}_u$ for $u\in A\cup B$; normalizing to
$\sum_{u}\theta^{(t+1)}_u=1$ recovers Eq.~\eqref{eq:bt_update} with $\varepsilon=0$,
restricted to cross-pairs.

\begin{proposition}[Scale invariance under bipartite comparisons]
For any $c>0$, replacing all cross-pair counts by $c\,\tau_{ab}$ (hence $n_{ab}\mapsto c\,n_{ab}$)
leaves the mapping $\theta^{(t)}\mapsto\theta^{(t+1)}$ in Eq.~\eqref{eq:bt_update} unchanged after normalization;
thus the fixed point $\hat\theta$ (up to scale) is unchanged.
\end{proposition}
\begin{proof}
Each $\text{wins}_u$ and each term of $c^{(t)}_u$ scales by $c$, so $\text{wins}_u/c^{(t)}_u$ is invariant;
normalization then gives the same $\theta^{(t+1)}$.
Equivalently, the bipartite BT log-likelihood
\begin{align}
\small
\mathcal{L}(\theta)
&= \sum_{(a,b)\in E}\!\Bigl[\tau_{ab}\log\theta_a+\tau_{ba}\log\theta_b \Bigr] \notag\\
&\quad - \sum_{(a,b)\in E}\! n_{ab}\log\!\bigl(\theta_a+\theta_b\bigr).
\end{align}
is multiplied by $c$ and has the same maximizer up to normalization.
\end{proof}

\begin{proposition}[MM minorization and monotone ascent for $\varepsilon=0$ (bipartite)]
\label{prop:mm_bt}
With the notation above, the update $\tilde\theta_u=\text{wins}_u/c^{(t)}_u$ for $u\in A\cup B$,
followed by normalization, is a Majorize–Minimize step for $\mathcal{L}$ and satisfies
$\mathcal{L}(\theta^{(t+1)})\ge \mathcal{L}(\theta^{(t)})$, with equality iff $\theta^{(t+1)}=\theta^{(t)}$.
\end{proposition}
\begin{proof}
{
\setlength{\jot}{3pt}
By convexity of $-\log x$, with $x=\theta_a+\theta_b$ and $x_0=S^{(t)}_{ab}$,
\(
-\log(\theta_a+\theta_b)\ge -\log S^{(t)}_{ab}-\tfrac{\theta_a+\theta_b-S^{(t)}_{ab}}{S^{(t)}_{ab}}.
\)
Multiplying by $n_{ab}$ and summing over $E$ yields
{\small                     % or \footnotesize / \scriptsize
\begin{align}
\mathcal{L}(\theta)
&\ge \sum_{(a,b)\in E}\!\bigl[\tau_{ab}\log\theta_a+\tau_{ba}\log\theta_b\bigr] \notag\\
&\quad - \sum_{(a,b)\in E}\! n_{ab}\,\log S^{(t)}_{ab}
      - \sum_{(a,b)\in E}\! \tfrac{n_{ab}}{S^{(t)}_{ab}}\bigl(\theta_a+\theta_b - S^{(t)}_{ab}\bigr) \notag\\
&= \sum_{u\in A\cup B}\!\text{wins}_u\log\theta_u
   - \sum_{u\in A\cup B}\!\theta_u\, c^{(t)}_u
   + \text{const}\bigl(\theta^{(t)}\bigr) \notag\\
&=: Q(\theta\,|\,\theta^{(t)}) + \text{const}.
\end{align}
}
}%
$Q$ is tight at $\theta^{(t)}$ and strictly concave on $\{\theta>0\}$; the first-order
conditions give $\text{wins}_u/\theta_u - c^{(t)}_u=0$, i.e., $\tilde\theta_u=\text{wins}_u/c^{(t)}_u$.
Since $\mathcal{L}(c\theta)=\mathcal{L}(\theta)$ for any $c>0$, normalizing $\tilde\theta$ to
$\sum_u\theta^{(t+1)}_u=1$ does not change $\mathcal{L}$; hence
{
\setlength{\jot}{1pt}
\begin{align}
\mathcal{L}(\theta^{(t+1)})
&\ge Q(\theta^{(t+1)}\,|\,\theta^{(t)}) + \text{const} \notag\\
&\ge Q(\theta^{(t)}\,|\,\theta^{(t)}) + \text{const}
= \mathcal{L}(\theta^{(t)}),
\end{align}
}%
with equality iff $\theta^{(t+1)}=\theta^{(t)}$.
\end{proof}

\begin{proposition}[DF-QuAD combination via signed product gap]
\label{prop:dfquad_proof}
Let $p$ be a node with supporters $\mathrm{S}(p)$ and attackers $\mathrm{A}(p)$.
Define
\begin{equation}
\alpha_p \;=\;
\prod_{a \in \mathrm{S}(p)} (1 - s_a)
\;-\;
\prod_{b \in \mathrm{A}(p)} (1 - s_b),
\label{eq:alpha}
\end{equation}
where $s_a,s_b\in[0,1]$ are propagated strengths.
Let $\tau_p'\in[0,1]$ be the BT-calibrated intrinsic weight.
Then the parent score
\begin{equation}
s_p \;=\;
\begin{cases}
\tau_p' \;-\; \alpha_p \,\tau_p', & \alpha_p > 0, \\[2pt]
\tau_p' \;-\; \alpha_p \,(1 - \tau_p'), & \alpha_p \le 0,
\end{cases}
\label{eq:linfluence}
\end{equation}
is exactly the DF\mbox{-}QuAD update for $p$ and satisfies $s_p\in[0,1]$, is continuous in $\alpha_p$, and is nonincreasing in $\alpha_p$.
\end{proposition}

\begin{proof}
Define the DF\mbox{-}QuAD noisy-OR aggregates for support and attack:
\begin{equation*}
v_s \;=\; 1 - \prod_{a\in \mathrm{S}(p)} (1-s_a), 
\qquad
v_a \;=\; 1 - \prod_{b\in \mathrm{A}(p)} (1-s_b)
\end{equation*}
Let the signed net-support gap be $\delta  := v_s - v_a \in [-1,1]$.
By~\eqref{eq:alpha},
\begin{align}
\delta &:= v_s - v_a \nonumber\\
&= \left(1-\prod_{a\in \mathrm{S}(p)}(1-s_a)\right)
 - \left(1-\prod_{b\in \mathrm{A}(p)}(1-s_b)\right) \nonumber\\
&= \prod_{b\in \mathrm{A}(p)}(1-s_b)
 - \prod_{a\in \mathrm{S}(p)}(1-s_a)
 \;=\; -\,\alpha_p. \nonumber
\end{align}
DF\mbox{-}QuAD’s combination function $C$ maps $(\tau_p',v_a,v_s)$ to
\begin{equation*}
C(\tau_p',v_a,v_s) \;=\;
\begin{cases}
\tau_p' + (1-\tau_p')\,\delta, & \delta \ge 0,\\[2pt]
\tau_p' + \tau_p'\,\delta,     & \delta < 0
\end{cases}
\end{equation*}
Substituting $\delta=-\alpha_p$ yields exactly~\eqref{eq:linfluence}:
if $\alpha_p\le 0$ (i.e., $\delta\ge 0$), then $s_p=\tau_p'-(1-\tau_p')\alpha_p$;
if $\alpha_p>0$ (i.e., $\delta<0$), then $s_p=\tau_p'-\tau_p'\alpha_p$.

\emph{Range.} Since both products in~\eqref{eq:alpha} lie in $[0,1]$, we have $\alpha_p\in[-1,1]$.
For $\alpha_p>0$, $s_p=\tau_p'(1-\alpha_p)\in[0,\tau_p']\subseteq[0,1]$.
For $\alpha_p\le 0$, $s_p=\tau_p'+(1-\tau_p')(-\alpha_p)\in[\tau_p',1]\subseteq[0,1]$.

\emph{Continuity.} At $\alpha_p=0$ both branches give $s_p=\tau_p'$, hence $s_p$ is continuous in $\alpha_p$.

\emph{Monotonicity.} For $\alpha_p>0$, $\frac{\partial s_p}{\partial \alpha_p}=-\tau_p'\le 0$; 
for $\alpha_p\le 0$, $\frac{\partial s_p}{\partial \alpha_p}=-(1-\tau_p')\le 0$.
Thus $s_p$ is nonincreasing in $\alpha_p$, so larger attack dominance (larger $\alpha_p$) never increases $s_p$.
\end{proof}

\section{Examples}
In this section, we provide some examples where pairwise tournaments and calibrations helped with correct predictions. Here, ART leverages \textit{Llama 3.3 70B Instruct} model is used with the \textit{intr.} setting. Increasing number of attack and support arguments in the tournament leads to more stable results with lower variance at the cost of higher latency. In Figure~\ref{fig:examples} we provide more examples where ART guides the final decision in the right direction using two support and attack arguments. 

\begin{figure*}
\centering
\includegraphics[width=\linewidth]{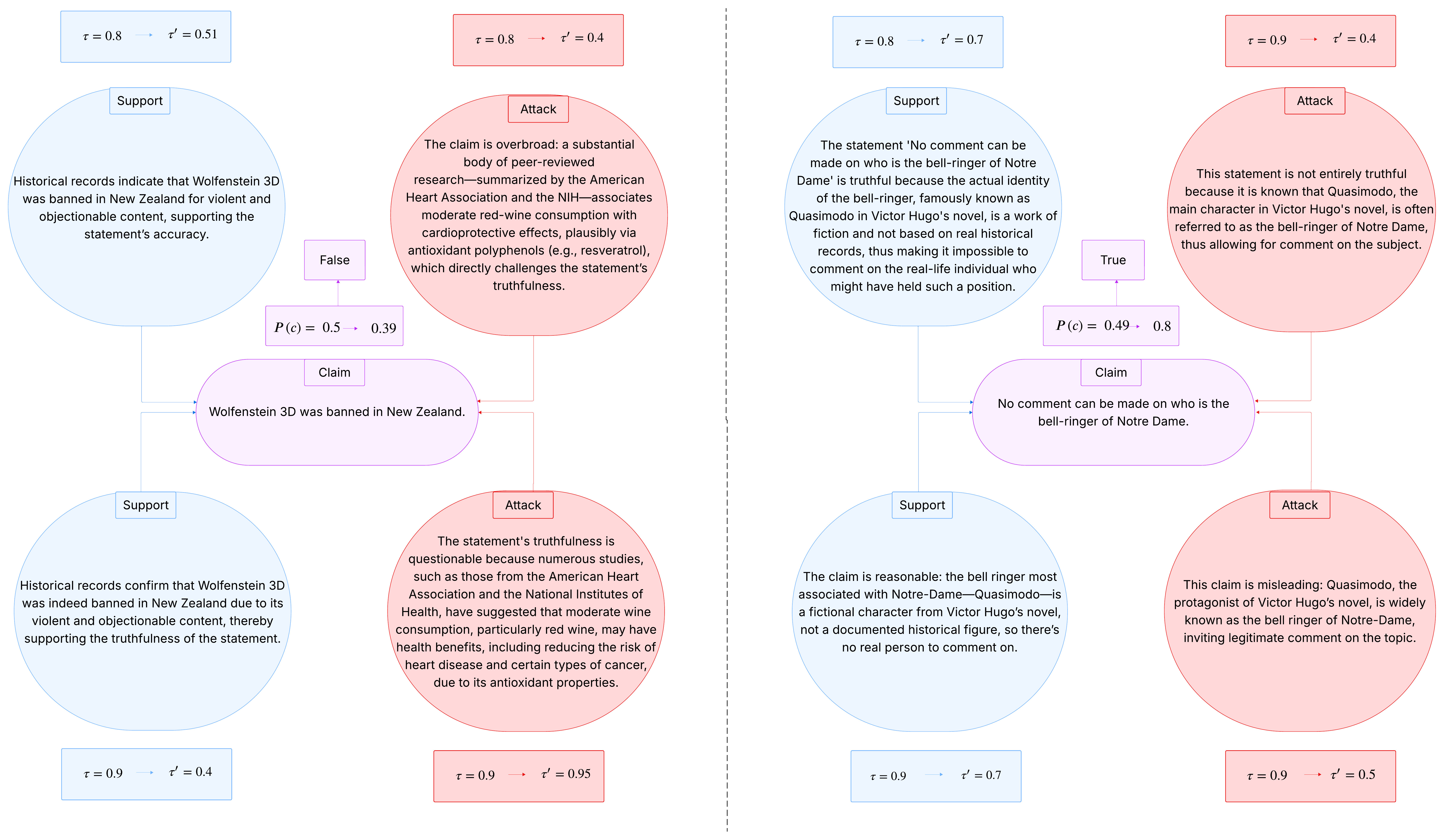}
\caption{Examples where ART improves the prediction results due to calibration that happens as a result of pairwise tournaments. Increasing the number of support and attack arguments in the tournament provides greater stability to the estimated strengths.}
\label{fig:examples}
\end{figure*}

% \section{LLM Optimization}

\section{Multi-Tree Ensemble}\label{sec:appendix_multi_tree}

In our exploration of tree-based reasoning structures, we hypothesized that an ensemble method might yield more robust and accurate results. To this end, we designed an experiment comparing three distinct configurations: a simple single-tree $(D=1,B=1)$, a single-tree with a pairwise tournament $(D=1,B=2)$, and a multi-tree ensemble. The ensemble method involved generating two independent $(D=1,B=1)$ Argumentative Reasoning Trees (ART) and averaging their final scores. This allowed us to test whether a direct, internal comparison of arguments is more effective than aggregating scores from isolated reasoning paths.

The results of this experiment, presented in Table~\ref{tab:multi_tree_ensemble}, reveal a surprisingly nuanced outcome where no single configuration is universally superior. The single-tree $(D=1,B=2)$ approach, which facilitates a direct pairwise tournament, proves to be the most robust baseline, achieving the highest accuracy in the majority of test cases across all datasets. This confirms that forcing a direct comparison between arguments is a powerful and generally effective strategy.

However, a deeper analysis shows that this advantage is not absolute and is~\textbf{highly dependent on the capability of the evaluator model}. This dependency creates clear and interesting patterns in the results. For instance, when a state-of-the-art model like~\textbf{Llama 3.3 70B} acts as the evaluator, simpler structures often become the top performers. We see the multi-tree ensemble winning with the Llama 3.3 70B generator on MedClaim, and the basic ($D=1,B=1$) tree winning with the same generator on StrategyQA. This suggests that a powerful evaluator may not need a complex reasoning structure to accurately assess a claim's validity.

\begin{table*}[h!]
 \centering
 \resizebox{\textwidth}{!}{
   \begin{tabular}{lllccc}
     \toprule
     \textbf{Dataset} & \textbf{Model} & \textbf{Evaluator} & 
     \makecell{\textbf{ART}\\\textbf{Single-Tree} ($D{=}1,B{=}1$)} &
     \makecell{\textbf{ART}\\\textbf{Single-Tree} ($D{=}1,B{=}2$)} & 
     \makecell{\textbf{ART}\\\textbf{Multi-Tree Ensemble($D{=}1,B{=}1$)}} \\
     \midrule
     \multirow{9}{*}{MedClaim} 
       & Llama 3.1 3B Instruct      & Llama 3.1 3B Instruct   & 53.68 & \textbf{56.42} & 53.68 \\
       & Llama 3.1 3B Instruct      & Llama 3.3 70B Instruct  & \textbf{65.05} & 60.84 & 61.2 \\
       & Llama 3.1 8B Instruct      & Llama 3.1 8B Instruct   & 56 & \textbf{56.21} & 55.37 \\
       & Llama 3.1 8B Instruct      & Llama 3.3 70B Instruct  & 63.57 & 62.97 & \textbf{63.67} \\
       & Mistral-7B Instruct-v0.2   & Mistral-7B Instruct-v0.2    & 52.84 & \textbf{54.31} & 52.58 \\
       & Mistral-7B Instruct-v0.2   & Llama 3.3 70B Instruct  & 57.68 & \textbf{58.31} & 57.89 \\
       & Mixtral-8x7B Instruct-v0.1 & Mixtral-8x7B Instruct-v0.1   & 56.21 & \textbf{61.05} & 60.63 \\
       & Mixtral-8x7B Instruct-v0.1 & Llama 3.3 70B Instruct  & 65.26 & \textbf{65.68} & 64.5 \\
       & Llama 3.3 70B Instruct     & Llama 3.3 70B Instruct  & 69.05 & 67.36 & \textbf{69.26} \\
     \midrule
     \multirow{9}{*}{StrategyQA} 
       & Llama 3.1 3B Instruct      & Llama 3.1 3B Instruct   & 57.6 & \textbf{59.4} & 59 \\
       & Llama 3.1 3B Instruct      & Llama 3.3 70B Instruct  & 62.4 & \textbf{64.4} & 61.8 \\
       & Llama 3.1 8B Instruct      & Llama 3.1 8B Instruct   & 59.2 & \textbf{64.6} & 60.2 \\
       & Llama 3.1 8B Instruct      & Llama 3.3 70B Instruct  & 71.4 & \textbf{72.2} & 70.8 \\
       & Mistral-7B Instruct-v0.2   & Mistral-7B Instruct-v0.2    & 58.4 & 61.2 & \textbf{61.6} \\
       & Mistral-7B Instruct-v0.2   & Llama 3.3 70B Instruct  & 70 & \textbf{70.2} & 68.6 \\
       & Mixtral-8x7B Instruct-v0.1 & Mixtral-8x7B Instruct-v0.1   & \textbf{63.6} & 61 & 61.2 \\
       & Mixtral-8x7B Instruct-v0.1 & Llama 3.3 70B Instruct  & 68.8 & 65.2 & \textbf{69.2} \\
       & Llama 3.3 70B Instruct     & Llama 3.3 70B Instruct  & \textbf{75.8} & 75.4 & 74.6 \\
       \midrule
      \multirow{9}{*}{TruthfulQA} 
       & Llama 3.1 3B Instruct      & Llama 3.1 3B Instruct   & 63.8 & \textbf{64.8} & 64.6 \\
       & Llama 3.1 3B Instruct      & Llama 3.3 70B Instruct  & 67 & \textbf{69} & 68.2 \\
       & Llama 3.1 8B Instruct      & Llama 3.1 8B Instruct   & \textbf{68.2} & 62 & 60 \\
       & Llama 3.1 8B Instruct      & Llama 3.3 70B Instruct  & \textbf{75.6} & 73 & 75 \\
       & Mistral-7B Instruct-v0.2   & Mistral-7B Instruct-v0.2    & 59.8 & 60 & \textbf{60.8} \\
       & Mistral-7B Instruct-v0.2   & Llama 3.3 70B Instruct  & 69.2 & \textbf{72.2} & 70.6 \\
       & Mixtral-8x7B Instruct-v0.1 & Mixtral-8x7B Instruct-v0.1   & 61.6 & \textbf{62} & 50.2 \\
       & Mixtral-8x7B Instruct-v0.1 & Llama 3.3 70B Instruct  & 78 & \textbf{79.4} & 71.4 \\
       & Llama 3.3 70B Instruct     & Llama 3.3 70B Instruct  & 79.4 & \textbf{80.2} & 80 \\
     \bottomrule
   \end{tabular}
 }
 \caption{\small Accuracy comparison between single-tree ($D{=}1, B{=}1$), single-tree ($D{=}1, B{=}2$) and multi-tree ensemble ART configurations with intrinsic strength calibration, broken down by generator and evaluator models . The results show that the preferred method is highly dependent on the evaluator's capability. Best scores in each row are in bold.}
 \label{tab:multi_tree_ensemble}
\end{table*}

This leads to our central insight: \textbf{the optimal reasoning architecture is not independent but is co-dependent on the components of the evaluation pipeline}. While increasing tree breadth for direct argument comparison remains a strong general-purpose strategy, our findings demonstrate that simpler architectures and ensemble methods can be unexpectedly potent. Their effectiveness, particularly when paired with highly capable evaluators, confirms that the interplay between the generator, the evaluator, and the reasoning structure is a critical factor in achieving maximum performance. 

Crucially, a practical consideration is the trade-off between accuracy and computational cost. We observe that in the specific scenarios where the Multi-Tree Ensemble outperforms other methods, the simpler Single-Tree $(D=1,B=1)$ configuration is often a close second-place contender. Given that the ensemble method effectively doubles the number of LLM calls and the associated overhead, the marginal performance gain frequently fails to justify the significant increase in cost and latency. This efficiency analysis provides a compelling final argument, reinforcing our conclusion that pursuing broader single-tree structures is a more promising and practical direction than adopting multi-tree ensembles.

\section{Time Complexity of ART}\label{sec:time_complexity}

\begin{proposition}[Counts on a full $2b$-ary tree of depth $d$]
Let each parent have $b$ supporters and $b$ attackers (branching factor $2b$), with root at depth $0$.
Then the total nodes $(N)$ and cross-group (support$\times$attack) pairs $(M)$ are
\begin{align*}
N &= \sum_{j=0}^{d} (2b)^j \;=\; \frac{(2b)^{d+1}-1}{\,2b-1\,},\\
M &= \sum_{j=0}^{d-1} (2b)^j\, b^2 \\
  &= b^2\,\frac{(2b)^d - 1}{2b - 1} \\
  &= \tfrac{b}{2}\,(N - 1), \quad
\end{align*}
\end{proposition}

% \begin{proof}[Sketch]
% Using $(2b)^d=\frac{(2b)^{d+1}}{2b}$, we have
% \begin{align*}
% M
%   &= b^2\,\frac{(2b)^d - 1}{2b - 1} \\
%   &= b^2\,\frac{\frac{(2b)^{d+1}}{2b} - 1}{2b - 1} \\
%   &= \frac{b}{2}\,\frac{(2b)^{d+1} - 2b}{2b - 1} \\
%   &= \frac{b}{2}\!\left(\frac{(2b)^{d+1} - 1}{2b - 1} - 1\right) \\
%   &= \tfrac{b}{2}\,(N - 1).
% \end{align*}
% \end{proof}

% \noindent\textbf{Setup.}
% Let the tree be full $2b$-ary with depth $d$ (root at $0$).
% The number of nodes and cross-group (support$\times$attack) pairs are
% \begin{align*}
% N &= \sum_{j=0}^{d} (2b)^j \;=\; \frac{(2b)^{d+1}-1}{2b-1},\\
% M &= \sum_{j=0}^{d-1} (2b)^j\,b^2 \;=\; b^2\,\frac{(2b)^d-1}{2b-1}
%    \;=\; \tfrac{b}{2}\,(N-1).
% \end{align*}

\paragraph{ART (pairwise + Bradley--Terry + DF-QuAD).}
Per parent there are $b^2$ cross-group pairs; across all internal nodes this is $M$.
With a fixed cap $k{=}100$ BT iterations (constant), one BT solve per parent costs
$O(b^2)$ per iteration, so the total BT work is $O(kM)=O(M)$.
Pairwise evaluation is $O(M)$ where an LLM judge is used to compare all cross pair arguments, and bottom-up DF\mbox{-}QuAD aggregation is $O(bN)$
($O(2b)$ per internal node). Since $M=\Theta(bN)$,
\begin{align*}
T_{\text{ART}} &= O(M) + O(M) + O(bN) + O(N)\\
               &= \Theta(bN) = \Theta\!\big(b(2b)^d\big).
\end{align*}
\emph{Space:} $O(N+M)=\Theta(bN)$ if storing all pairwise tallies $\tau_{ab}$;
with streaming BT statistics, peak space can be reduced to $O(N)$.

\paragraph{ArgLLM (no pairwise, no Bradley--Terry).}
ArgLLM performs tree generation, intrinsic scoring, and DF\mbox{-}QuAD only.
Aggregation is $O(bN)$ overall; other steps are $O(N)$:
\begin{align*}
T_{\text{ArgLLM}} &= O(bN) + O(N) \\
                  &= \Theta(bN)
                  \;=\; \Theta\!\big(b(2b)^d\big).\\
\textit{Space}    &= O(N)
\end{align*}

\paragraph{Comparison.}
Asymptotically both are $\Theta(bN)$ (BT iterations capped),
but ART carries a larger leading constant due to pairwise competitions and BT calculations and typically higher space.
% Let $c_{\text{pair}}$ be the per-pair evaluation cost, $c_{\text{bt}}$ the BT
% cost per pair (absorbing the fixed $k$), and $c_{\text{df}}$ the DF\mbox{-}QuAD
% cost per child. A simple model gives
% \begin{align*}
% T_{\text{ArgLLM}} &\approx c_{\text{df}}\,b\,N,\\
% T_{\text{ART}}    &\approx (c_{\text{pair}}+c_{\text{bt}})\,M + c_{\text{df}}\,b\,N.
% \end{align*}
% Using $M=\tfrac{b}{2}(N-1)$,
% \[
% \frac{T_{\text{ART}}}{T_{\text{ArgLLM}}}
% \;\approx\;
% 1 + \frac{(c_{\text{pair}}+c_{\text{bt}})}{2\,c_{\text{df}}}
% \;>\; 1,
% \]
so ART is slower by a constant factor (and uses $\Theta(bN)$ space if pairwise
tallies are retained), whereas ArgLLM is lighter in both time constant and space. 

ART parallelizes naturally: each supporter–attacker pair can be scored independently, and each parent’s Bradley–Terry solve runs independently as well. We process the tree level-by-level (all nodes at the same depth in parallel) and batch pair evaluations on \textit{A100} GPU. With this task- and data-parallel structure, ART achieves runtime comparable with ArgLLM.

\section{Prompts}
\label{appendix:prompts}
This section details the prompts employed in our experiments. We utilized distinct templates for several key tasks: Direct Questioning, Chain-of-Thought, generating supporting and attacking arguments, assigning initial scores via a dedicated scoring prompt, and conducting the pairwise comparisons for calibration.

\textbf{Direct Questioning Prompt.}
The prompt in Figure \ref{fig:direct-prompt} was used to generate the direct question for a given claim.

\begin{figure}[!t]
\centering
\includegraphics[width=\linewidth]{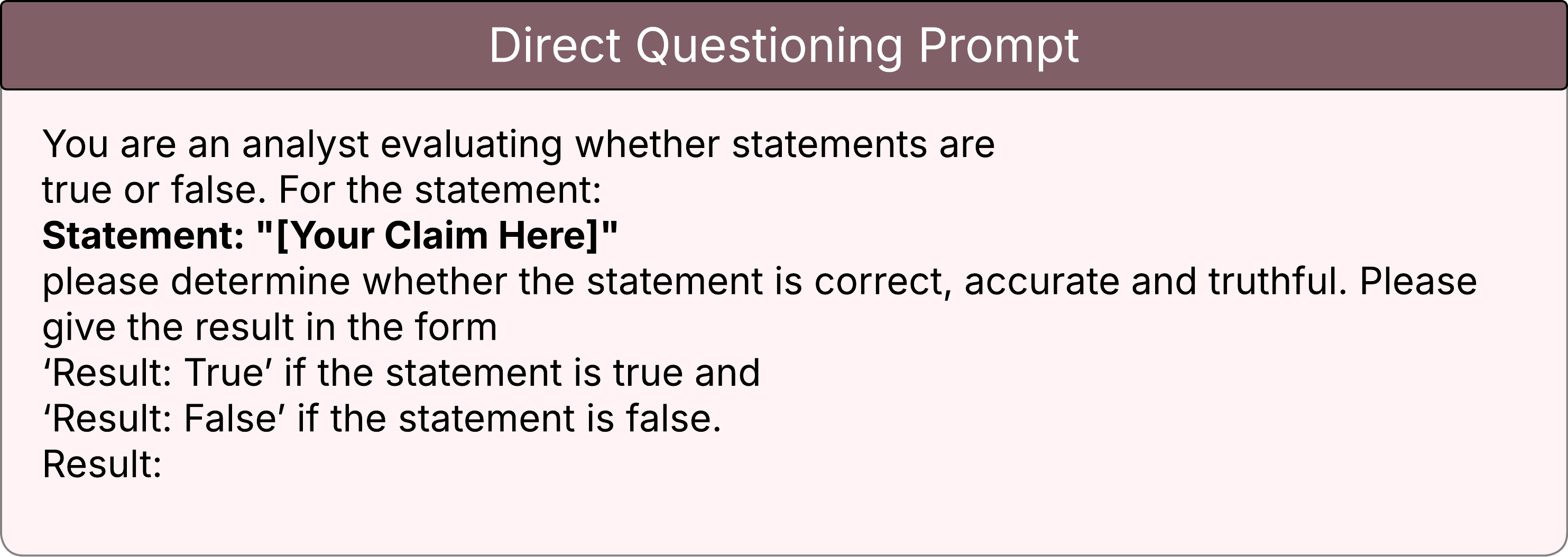}
\caption{Direct prompt.}
\label{fig:direct-prompt}
\end{figure}
\textbf{Chain-of-Thought Prompt.} The prompt in Figure \ref{fig:cot} was used to generate the Chain-of-Thought based reasoning response for a given claim. 

\begin{figure}[!t]
\centering
\includegraphics[width=\linewidth]{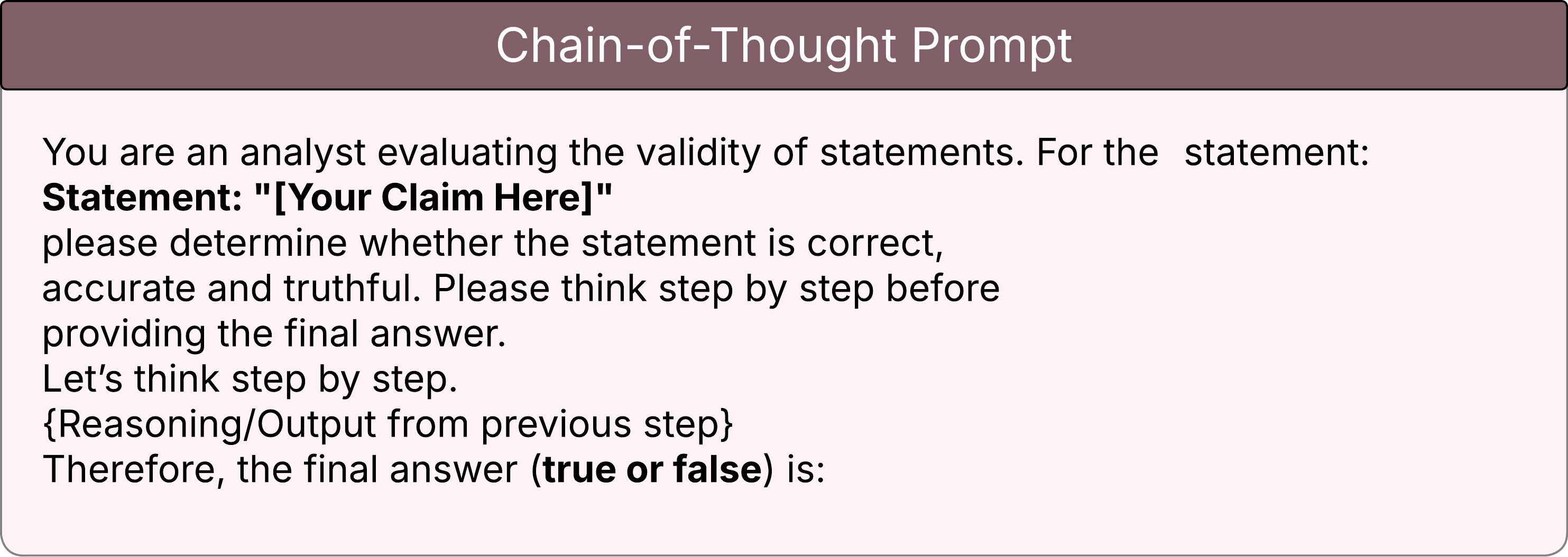}
\caption{Chain-of-thought prompt.}
\label{fig:cot}
\end{figure}

\noindent\textbf{Dynamic Argument Generation Prompt}
~The prompt in Figure \ref{fig:dynamic-arg-gen-prompt} is designed to dynamically generate either a supporting or an attacking argument for a given claim.

\begin{figure}[!t]
\centering
\includegraphics[width=\linewidth]{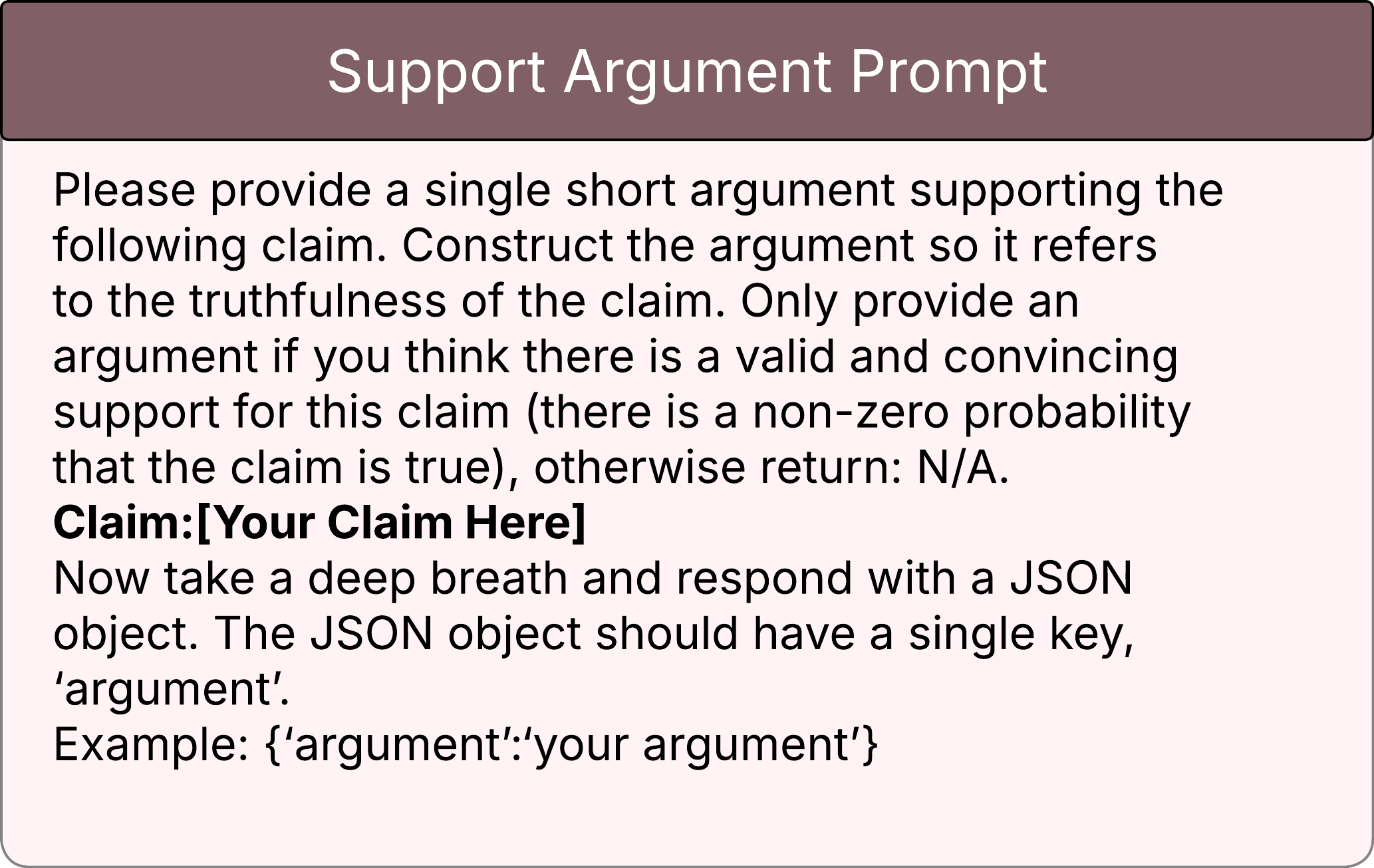}
\caption{Dynamic argument generation prompt. The example is for support argument generation and similarly for the attack argument `supporting' is changed into `attacking' and `support' into `attack'.}
\label{fig:dynamic-arg-gen-prompt}
\end{figure}
\newpage

\noindent\textbf{Dynamic Uncertainty Evaluator Prompt.}
Prompt in Figure \ref{fig:uncertainity-eval-prompt} is designed to dynamically generate Uncertainty scores using LLM as a judge for a support or attack argument. The prompt demonstrates a support Uncertainty Evaluation Prompt, where the language is framed in terms of an argument being in favour of a claim and how it supports the parent statement. For constructing the corresponding attack version, the wording is systematically adjusted: the phrase “in favour of” is replaced with “against” to reflect opposition, and “supports” is replaced with “refutes” to indicate contradiction. This ensures that the structure of the prompt remains consistent while capturing the intended argumentative stance.

\begin{figure}[!t]
\centering
\includegraphics[width=\linewidth]{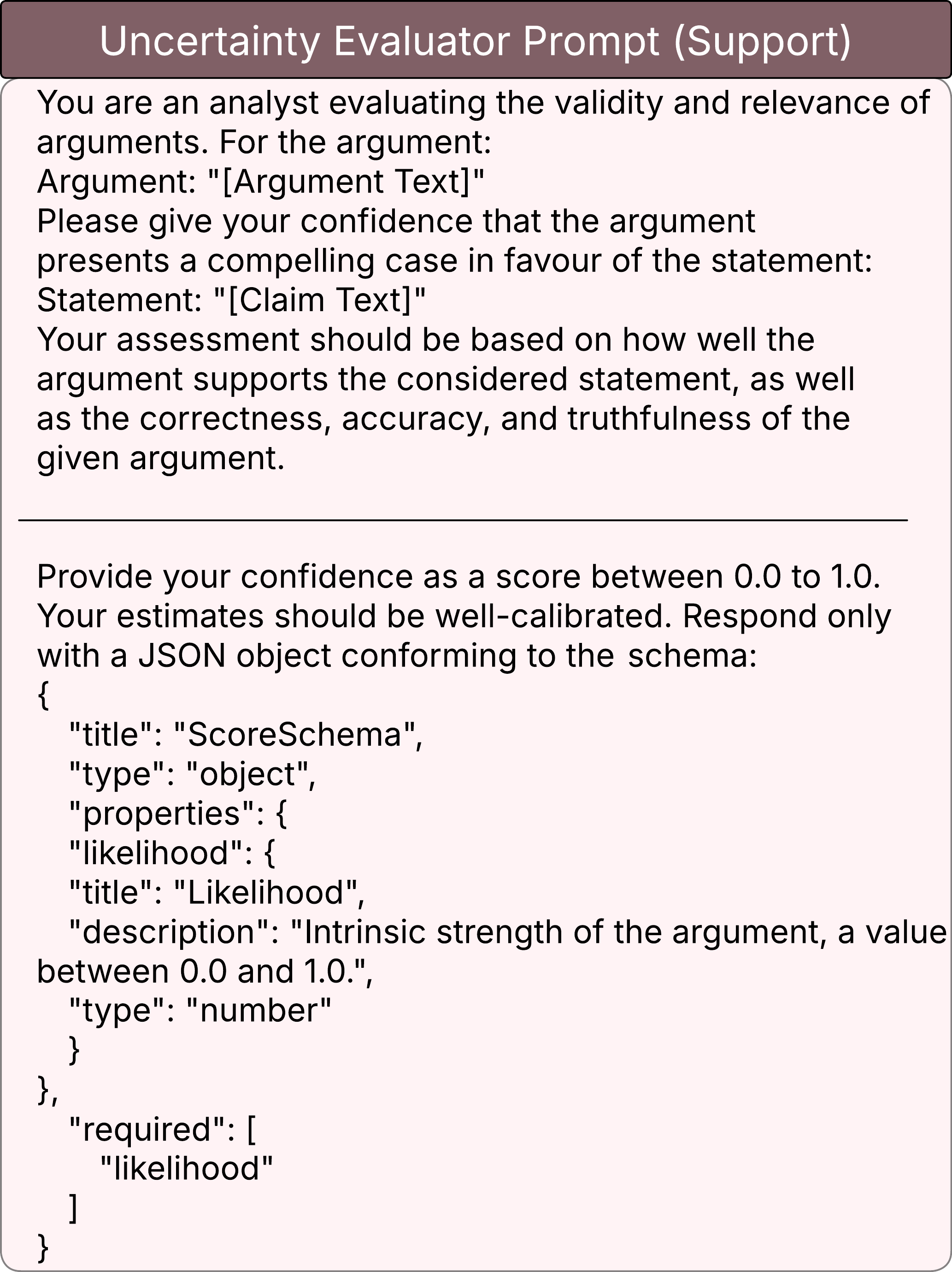}
\caption{Uncertainty evaluator prompt.}
\label{fig:uncertainity-eval-prompt}
\end{figure}

\paragraph{Pairwise Argument Comparison Prompt}
Figure \ref{fig:pairwise-comparison-prompt} illustrates the prompt designed to compare supporting and attacking arguments for a child node in ART. The judge LLM $J$ receives the original statement (the parent node) along with its two child nodes—one supporting and one attacking—and determines which argument is stronger. If both arguments are deemed equally strong, the model should return TIE.

\begin{figure}[!t]
\centering
\includegraphics[width=\linewidth]{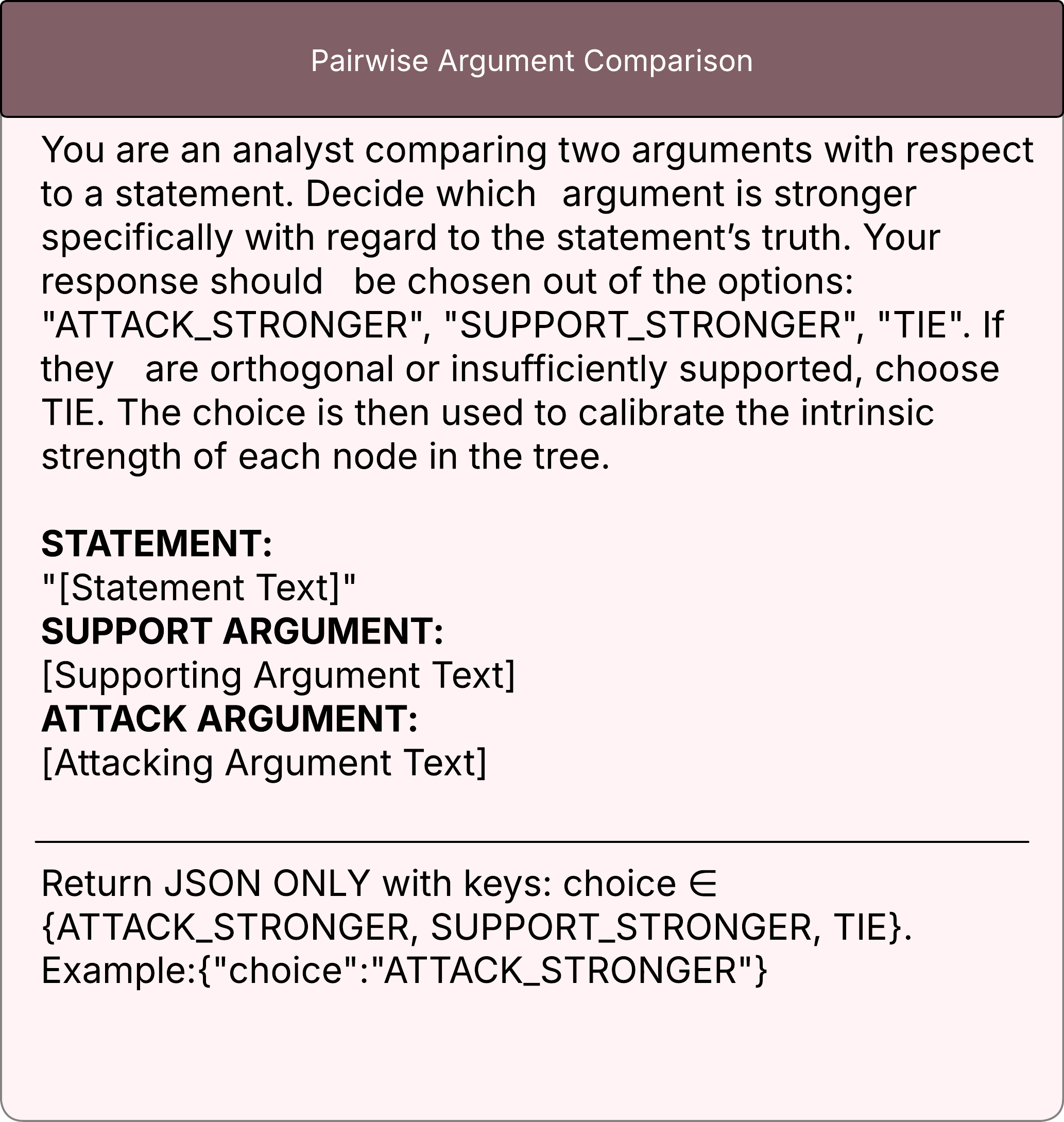}
\caption{Pairwise comparison prompt.}
\label{fig:pairwise-comparison-prompt}
\end{figure}

\begin{footnotesize}
% Requires: \usepackage[ruled,vlined]{algorithm2e}
\begin{algorithm*}[!t]
\caption{Calibration via LLM Pairwise Judgments + Bradley--Terry (bipartite)}
\label{alg:calibrate_bt_clearW}
\KwIn{Argument tree $T$ (root id \texttt{db0}), judge LLM $\mathcal{J}$, pairwise prompt $\Pi$, stabilizer $\varepsilon>0$, $\lambda \in [0,1] $}
\KwOut{Calibrated tree $T$}

\textbf{I. Identify calibration parents.}

$root \gets T[\texttt{db0}]$\;
\If{$root=\varnothing$}{
  \Return $T$\;
}
$C \gets \{\,p \in T \mid p.\texttt{supporters}\neq\varnothing \land p.\texttt{attackers}\neq\varnothing\,\}$\;
\If{$C=\varnothing$}{
  \Return $T$\;
}

\textbf{II. Per-parent setup and pairwise counts.}

\ForEach{parent $p \in C$}{
 $S \gets p.\texttt{supporters}, A \gets p.\texttt{attackers}$, 
$U_p \gets S \cup A$, 
$E_p \gets S \times A$

Initialize sparse $W_p[x,y]\gets 0$ for all $(x,y)\in E_p \cup \{(y,x):(x,y)\in E_p\}$\;

  \ForEach{$s \in S$}{
    \ForEach{$a \in A$}{
      $prompt \gets \Pi(p,\ s,\ a)$\;
      
      $winner \gets \mathcal{J}(prompt)$\;
      
      \uIf{$winner=\texttt{SUPPORT}$}{$W_p[s,a] \gets W_p[s,a] + 1$}
      
      \uElseIf{$winner=\texttt{ATTACK}$}{$W_p[a,s] \gets W_p[a,s] + 1$}
      
      % \Else{\tcp*[f]{tie: no update}}
    }
  }
}
\textbf{III. BT fitting and rescaling (per parent).}

\ForEach{parent $p \in C$}{
  \If{$E_p=\varnothing$}{\textbf{continue}\;}
  Initialize $\theta_u \gets 1$ for each $u \in U_p$\;

  \Repeat{convergence}{
    \ForEach{$u \in U_p$}{
      \uIf{$u \in S$}{
        $wins \gets \sum_{a \in A} W_p[u,a]$\;
        
        $den \gets \sum_{a \in A} \dfrac{W_p[u,a]+W_p[a,u]}{\theta_u+\theta_a}$\;
      }\Else(\tcp*[f]{$u \in A$}){
        $wins \gets \sum_{s \in S} W_p[u,s]$\;
        
        $den \gets \sum_{s \in S} \dfrac{W_p[u,s]+W_p[s,u]}{\theta_u+\theta_s}$\;
      }
      $\theta_u \gets \dfrac{wins}{den+\varepsilon}$\;
    }
    Normalize $\sum_{u \in U_p}\theta_u = 1$\;
  }

  \ForEach{$u \in U_p$}{
    $u.\texttt{initial\_weight} \gets \text{clip}\big((u.\texttt{initial\_weight}\cdot (1-\lambda) + \lambda \cdot \theta_u),\ 0,\ 1\big)$\;
  }
}

\Return $T$
\end{algorithm*}
\end{footnotesize}